\documentclass[12pt,a4paper]{icicel}

\def\currentvolume{11}
\def\currentissue{10}
\def\currentyear{2017}
\def\currentmonth{October}
\def\ppages{DOI: 10.24507/icicel.11.10.1515}

\usepackage{amsmath}
\usepackage{float}
 \usepackage[]{algorithm2e}
\usepackage{amssymb}
\usepackage{url}
\newtheorem{corollary}{Corollary}[section]

\newtheorem{proposition}{Proposition}[section]

\newtheorem{definition}{Definition}[section]
\newtheorem{example}{Example}[section]

\usepackage{epsfig}
\usepackage{epsf}
\usepackage{graphicx}
\usepackage{wrapfig}

\title[ICIC Express Letters, Vol.11, No.10, 2017]
      {On Evaluating the Quality of Rule-Based Classification Systems}

\author[Nassim Dehouche]{}

\begin{document}

\maketitle

\centerline{\scshape  Nassim Dehouche$^1$}
 \medskip
{\footnotesize
\centerline{$^1$Business Administration Division}
\centerline{Mahidol University International College} 
\centerline{Salaya, 73170, Thailand}
\centerline{Nassim.deh@mahidol.ac.th} }
\medskip

\centerline{Received; accepted }


\medskip

\begin{abstract}

{\em Two indicators are classically used to evaluate the quality of rule-based classification systems: \textit{predictive
accuracy}, i.e. the system's ability to successfully reproduce learning data and \textit{coverage}, i.e. the proportion of possible cases for which the logical rules constituting the system apply.
In this work, we claim that these two indicators may be insufficient, and additional measures of quality may need to be developed. We theoretically show that classification systems presenting "good"
predictive accuracy and coverage can, nonetheless, be trivially improved and illustrate this proposition
with examples.
To conceptualize our main claim, we characterize a property of reducibility. A classification
system is said to be reducible, if and only if, its constituent rules can be replaced by a subset
of their elementary conditions, while preserving the quality of the system. We derive a time-efficient
constructive algorithm to test this property and to improve a system's predictive accuracy and coverage in case of a positive
response. Furthermore, we provide a set of sufficient conditions that can be used to detect non-reducibility
and thus validate rule-based classification systems. We use the proposed approach to evaluate a previously published work applied to a public dataset pertaining to the business bankruptcy prediction, using three popular machine learning approaches (namely Genetic Algorithms, Inductive learning
and Neural Networks). The results of this application
support our main claim. We conclude this paper by suggesting that a classification system’s ability to clarify trade-offs between attributes should be measured, and used as an additional performance indicator. A possible further development of this work consists in developing such an indicator.\\}

\end{abstract}

\section{Introduction}

Machine learning is defined as the study of computational methods for improving performance by mechanizing the acquisition of knowledge from experience \cite{simon}. As data collection and storage become easier and cheaper, there has been a growing interest in machine learning to facilitate the extraction of quality knowledge from databases and numerous techniques have been proposed to perform this task. A large proportion of knowledge discovery and extraction problems can be formulated as classification problems. This work is concerned with the tasks of evaluating the accuracy and coverage of such techniques and is organized as follows. Section 2. reviews some relevant publications to this research\footnote{This article is the full version of an abstract entitled "On evaluating the quality of machine learning classification methods" presented at The Second International Conference on Mathematics and Statistics (AUS-ICMS '15), American University of Sharjah, Al Sharjah, UAE; 04/2015. }. Section 3. formally states the problem of evaluating the quality of a rule-based classification system, and section 4. introduces the algorithm we propose for the validation and reduction of such systems. This algorithm is applied to a real case pertaining to business bankruptcy prediction, in section 5., highlighting the need for validation and reduction before deploying a rule-based classification system. Finally, section 6. concludes this work with perspective for further development of the proposed approach.

\section{Related work}
The problem of inducing a minimal classification system being NP-Hard, it is generally tackled through the use of heuristic or meta-heuristic approaches \cite{kim,BELACEL}. The field of data mining has also seen several contributions relevant to this problem. We can notably cite approaches based on closure systems \cite{Carpineto}, for unlabeled data, which aim at compacting a space of descriptions into a reduced system of relevant sets that conveys as much of the information contained in the complete space, as possible. Such approaches have also been successfully adapted to labeled data \cite{Garriga}, which makes them utilizable for classification tasks. Considerations of coverage, accuracy and compactness are typically at the heart of the design of these algorithms and are used as metrics to evaluate their quality. A popular approach in data mining consists, for instance, in identifying free (or generator) item-sets \cite{Rosa}, which are expected to generate minimal sets of descriptions, or non-redundant sub-groups \cite{Mario} \cite{Ruping}. However, it should be noted, that these considerations are inherent to the learning approach used but can not be embedded in other methods, such as meta-heuristics. In this work however, we propose a generic approach that can be used to reduce a classification system after rules have been learned. Moreover, our main concern is not syntactical redundancy but unnecessary specialization in the structure of classification rules. Thus, given a classification system, we aim at testing the reducibility of such a system and when possible generating a more compact set of rules that would improve both its the accuracy and coverage, as we shall demonstrate, quite dramatically, on the \textit{Bankruptcy Predicion} dataset from the UCI machine learning repository, that originated from a previously published work \cite{kim}, and that continues to be used as a benchmark in more recent works \cite{martin1,sun,cheng,koklu}. The present work can also be situated close to the literature on pruning \cite{Cohen}, with the notable distinction however that learning errors we aim at correcting do not originate from noisy or conflicting data but from the adoption of a blind optimization
approach, with accuracy and coverage as objectives. As we shall see through the aforementioned example, such inconsistencies can appear in trivial data-sets, the triviality of which can, in fact, be hidden by the the unnecessary complexity of the resulting classification systems.

\section{Problem Statement}
\begin{wrapfigure}{L}{0.5\textwidth}
\centering
\begin{center}

\includegraphics[height=338pt]{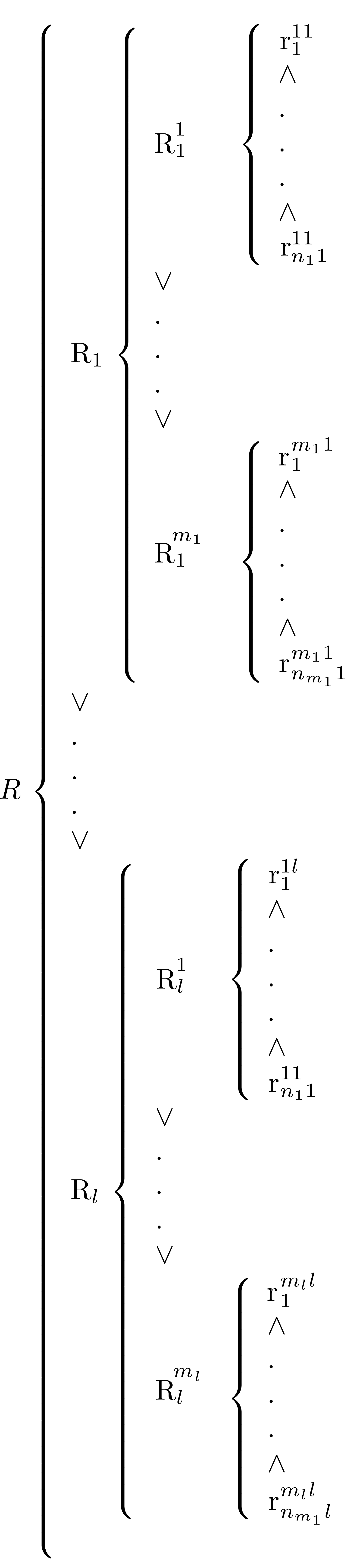}

\caption{\label{representation}Structure of a classification system} 

\end{center}
\end{wrapfigure}
Given a set $A$ of objects described by a set $H$ of attributes, and a set of classes $C=\{C_1,\dots,C_l\}$, a classification problem consists in inducing a set $R$ of logical rules, also known as a rule-based classification system, or an expert system, with an ability to assign objects from $A$ into classes from $C$. Attributes in $H$ can be of a qualitative or quantitative (discrete or continuous) nature.

Moreover, the domains of these attributes and the classes in $C$ can be non-ordered, in which case we would say that the classification problem if purely nominal, or there can exist a preference order among some of these domains, or between some classes. In this case the classification problem is said to be ordinal \cite{drsa}.

Let $H=\{h_1, \dots, h_q,\dots h_Q\}$ be the set of attributes considered to describe the objects in $A$, and $D=D_1 \times \dots \times D_q\times \dots D_Q$ the set of all possible descriptions for an object, each attribute $h_q\in H$ taking its values from domain $D_q$. \\

Let $R$ be a conjunctive normal form (CNF) classification system whose structure is presented in Figure \ref{representation}. In this set of logical rules, an overall rule $R_k,k\in\{1,\dots,l\}$ corresponds to each class $C_k$. An object can be assigned to any class, if it satisfies its corresponding overall rule. Consequently, we have $R=R_1\lor \dots\lor R_k \lor\dots \lor R_l$.

Each overall rule $R_k$ represents a disjunction of $m_k$ CNF assignment rules $R_k^{j},j\in\{1,\dots,m_k\}$, i.e. $R_k=R_k^{1}\lor \dots R_k^{m_k}, k\in \{1,\dots,l\}$.

Finally, each assignment rule $R_k^{j}$ consists of a conjunction of $n_k^{j}$ elementary conditions $r_k^{ji}, i\in \{1,\dots,n_k{j}\},j\in \{1,\dots,m_k\}, k\in\{1,\dots,l\}$, i.e. $R_k^{j}=r_k^{j1}\land\dots \land r_k^{j{n_k^{j}}}$. We can also view an assignment rule as a set of elementary condition. An elementary condition $r_k^{ji}, i\in \{1,\dots,n_k{j}\},j\in \{1,\dots,m_k\}, k\in\{1,\dots,l\}$ is of the form ``if \textit{condition} then $Class=C_k$'', the \textit{condition} part of each elementary condition being of the general form $h_q\in V$, with $q\in \{1,\dots,Q\}$, and $V\subset D^q$. The meaning of such a condition is ``attribute $h^q$ takes a value in domain V''.\\

In inductive machine learning a classification system $R$ is to be induced from a reference set $A' \subset A$ of objects whose classes are known. It should be noted that this reference set is often split into a \textit{training} set, i.e. assignments used for induction, and a \textit{test} set, i.e. assignments used for evaluation. However, this distinction is immaterial in the present work. Thus, we shall consider, without loss of generality, the whole of set $A'$ to be the training set, and the whole of set $A$ to be the test set. 

However, before deploying a classification system, we need to be able to quantify its performance and the quality of the classification it produces. The evaluation of competing machine learning methods for the induction of classification Systems has thus received more and more attention and numerous measures of quality exist for evaluating classification system \cite{simon}. However, they are typically based on the two following indicators \cite{kim} to be maximized. 
\begin{itemize}
\item \textit{Predictive Accuracy:} The percentage of objects in $A'$ that are correctly classified by the classification system $R$. This is an indicator of the ability of  a classification system to reproduce learning examples.
\item \textit{Coverage:} The proportion of objects in $A$ to which at least one overall rule in $R$ applies. This is a an indicator of a classification system's universality.
\end{itemize}

By virtue of the principle of \textit{Occam's Razor}, and for approximately equal values of the previous two indicators, the \textit{Compactness} of a classification system stemming from the principle of \textit{Occam's Razor}, is also considered as a measure of quality, albeit a less important and less formal one. This leads to favor simpler classification systems. In other word sets of logical rules that contain the least overall rules, assignment rules and elementary conditions are preferable.

\section{Proposed validation and reduction algorithm}

\begin{definition}
A classification system $R$ is said to be improvable if and only if, it contains an assignment rule $R_k^{j}=r_k^{j1}\land\dots \land r_k^{j{n_k^{j}}}$ that can be replaced by a strict subset of its elementary conditions, without degrading the accuracy and coverage of the system. In this case, we would say that assignment rule $R_k^{j}$ is reducible.   
\end{definition}

In other words, we are referring to classification systems in which some elementary conditions can be removed from an assignment rule, thus resulting in a more compact system with better coverage, since the conditions to satisfy an assignment rule are relaxed. Such a classification system would be considered improved if this operation does not degrade its predictive accuracy. 
\begin{proposition}
An assignment rule $R_k^{j}=r_k^{j1}\land\dots \land r_k^{j{n_k^{j}}}$ is reducible to a strict subset of its elementary conditions $r_k^{ji_1}\land\dots \land r_k^{ji_S}$, with $\{i_1,\dots,i_S\}\subset \{1,\dots,{n_k^{j}}\}$, if and only if, these elementary conditions are mutually exclusive with the condition part of any rule, assigning to a different category than $C_k$.
\end{proposition}
\begin{proof}
Since we want to remove some elementary conditions from an assignment rule, we must make sure that the remaining reduced subset of elementary conditions does not overlap with the condition part a rule assigning to a different category. Indeed, that would create a conflicting assignment, which would reduce the overall accuracy of the classification system. Moreover, the fact that we are thus expanding the condition part of an assignment rule, the objects previously satisfying this rule remain covered. Therefore, this operation cannot degrade the previous coverage of the classification system.
\begin{itemize}
\item Mutual exclusiveness implies that reducing an assignment rule can only lead to covering previously uncovered cases, thus improving coverage without degrading accuracy (and typically improving it as well).
\item Mutual non-exclusiveness implies that reducing an assignment rule leads to the creation of contradictory classification for some previously accurately classified cases, thus reducing accuracy.
\end{itemize}
\end{proof}
Algorithm \ref{alg} can be directly deduced from \textit{property 1}. It verifies the reducibility of the assignment rules constituting a classification system, and trivially improve it in case of a positive answer. 
\begin{algorithm}[h!]

 \KwData{A classification system $R$}
 \KwResult{Validating or else trivially improving $R$}

\For{each overall rule $R_k \in R$} 
{
\For{each assignment rule $R_k^{j}\in R_k$ }
{
\For{each elementary condition $r_k^{ji} \in R_k^{j}$ }
{
$PossibilityReduction\leftarrow True$;
\For{each overall rule $R_t \in R$ }
{

   \eIf{$t \neq k$} 
{
     
  \eIf{$R_k^{j} - r_k^{ji} \cap R_t \neq\emptyset$}   
    { 
     $PossibilityReduction\leftarrow False$;
     
     }

     }

}
  \eIf{$PossibilityReduction= True$}   
    { 
    //Elementary condition $r_k^{ji}$ can be removed from assignment rule $R_k^{j}$.\\
$R_k^{j}\leftarrow R_k^{j}- r_k^{ji}$;   
     
     } 
     
     } 

}

}
 
 \caption{\label{alg}Algorithm to verify reducibility and trivially improve classification systems}
\end{algorithm}

The following negative test can also be deducted from property 1.
\begin{corollary}
For every pair of rules assigning to different classes, an
attribute that does not appear in both rules cannot be part
of a clause or conjunction of clauses that would reduce the
expert system.
\end{corollary}
The following example shows illustrates the property of reducibility and the way the approach we propose can be used to identify unnecessarily complex classification systems and to reduce them to their simplest form,  while improving their quality indicators. 

\begin{example}
In order to give an intuitive understanding of how \textit{property} $1$ applies to a classification system, we consider a toy binary classification problem in which rules classifying objects photographed by a camera at the entrance of a building are to be induced. The photographed objects are to be recognized as being a car of or not being a car. Table \ref{toydata} presents the reference data set for this problem.

\begin{table}
\centering

\begin{center}
\begin{tabular}{|c|c|c|c|}
\hline
\textbf{Object}&\textbf{Number of passengers $(P)$}&\textbf{Number of wheels $(W)$}&\textbf{Class}\\
\hline

1 &$2$&  $3$ or less& Not car\\
2& $20$& $4$ or more& Not car\\
3& $2$& $4$ or more& Car\\
4& $12$& $4$ or more & Not Car\\
5& $1$&	 $3$ or less& Not car\\

\hline
\end{tabular}
\end{center}
\caption{\label{toydata}Illustrative data set}
\end{table}
Let us suppose that the classification system presented in table \ref{toyrules} is induced. It can be noted that objects 1 to 4 in table \ref{toydata} are accurately classified by this set of rules, but object 5 satisfies no assignment rule. Thus the predictive accuracy and coverage of the classification system, with regards to this data set are both $\frac{4}{5}=80 \%$. According to \textit{property} $1$., assignment rule $R_1^{1}$ can be reduced to the elementary condition ``if $W=3$ or less then Not car'', this elementary condition being mutually exclusive with the condition part of rules $R_2^{1}$. For the same reason, rule $r_1^{2}$ can be reduced to elementary condition ``if $P>10$ then Not car'', whereas rule $R_2^{1}$ is not reducible, because each one of its elementary conditions covers some objects that are also covered by the condition parts of rule $R_1^{1}$ and rule $r_1^{2}$. Thus the previous classification system can be trivially improved, as shown in table \ref{improve}, resulting in $80 \%$ accuracy and coverage.  Section 5. proposes a similarly-structured albeit more realistic application.

\begin{table}
\centering

\begin{center}
\begin{tabular}{|c|c|c|}
\hline
\textbf{Assignment rule}&\textbf{Elementary conditions}&\textbf{Class}\\
\hline

$R_1^{1}$ &$P>1$ and $W=3$ or less&   Not car\\
$R_1^{2}$ &$P>10$ and $W=4$ or more&   Not car\\
$R_2^{1}$ &$P=2$ and $W=4$ or more&   Car\\
\hline
\end{tabular}
\end{center}
\caption{\label{toyrules}Illustrative classification system}
\end{table}

\begin{table}
\centering

\begin{center}
\begin{tabular}{|c|c|c|}
\hline
\textbf{Assignment rule}&\textbf{Elementary conditions}&\textbf{Class}\\
\hline

$R_1^{1}$ &$W=3$ or less&   Not car\\
$R_1^{2}$ &$P>10$ &   Not car\\
$R_2^{1}$ &$P=2$ and $W=4$ or more&   Car\\
\hline
\end{tabular}
\end{center}
\caption{\label{improve}Reduced form of the illustrative classification system}
\end{table}
\end{example}

\section{Application: The discovery of experts’ decision rules from qualitative bankruptcy data}This applicative section is based on a work published in 2003 \cite{kim}, which notably proposes a genetic algorithm-based data mining method for discovering bankruptcy decision rules from experts’ qualitative decisions. Data related to this article were made public on the UCI machine learning repository, in 2014 \cite{uci} and several ulterior works have built upon its results \cite{martin1,sun,cheng}, the most recent of which being \cite{koklu}, which uses the same data-set. As we shall explain in this section, this state of affair is unfortunate, as this problem and data-set are trivial. However, the triviality of this problem has been masked by the use of a blind optimization approach (meta-heuristic) for the classical optimization of the accuracy and coverage. 
In this problem, a group of experts evaluates 772 manufacturing and service companies, through 6 economic factors:

\begin{itemize}
\item  Industry Risk (IR).
\item  Management Risk (MR).
\item  Financial Flexibility (FF).
\item  Credibility (CR).
\item  Competitiveness (CO).
\item  Operating Risk (OP).
\end{itemize}
Possible values for each factor are Negative (N), Average (A), and Positive (P) and two classes are observed Bankruptcy (B) and Non-bankruptcy (NB).

\begin{table}[h!]
\begin{center}
\resizebox{\columnwidth}{!}{
\begin{tabular}{|c|l|}

\hline
\textbf{Assignment rule}&\textbf{Elementary conditions}\\
\hline
Rule1 &\small{IF FF is positive and CO is positive THEN Nonbankrupt}\\
Rule2 &\small{IF FF is positive and CO is average and CR is average
or positive THEN Nonbankrupt}\\
Rule3 &\small{IF FF is positive and CO is average and CR is negative
THEN Bankrupt}\\
Rule4 &\small{IF FF is positive and CO is negative and MR is average
or positive THEN Nonbankrupt}\\
Rule5 &\small{IF FF is positive and CO is negative and MR is negative
THEN Bankrupt}\\
Rule6 &\small{IF FF is average and MR is positive and CO is average
or positive THEN Nonbankrupt}\\
Rule7 &\small{IF FF is average and MR is positive and CO is negative
THEN Bankrupt}\\
Rule8 &\small{IF FF is average and MR is average and OP is average
or positive THEN Nonbankrupt}\\
Rule9 &\small{IF FF is average and MR is average and OP is negative
THEN Bankrupt}\\
Rule10 &\small{IF FF is average and MR is negative THEN Bankrupt}\\
Rule11 &\small{IF FF is negative and OP is positive THEN Nonbankrupt}\\
Rule12 &\small{IF FF is negative and OP is average and IR is average
or positive THEN Nonbankrupt}\\
Rule13 &\small{IF FF is negative and OP is average and IR is negative
THEN Bankrupt}\\
Rule14 &\small{IF FF is negative and OP is negative and CR is average
or positive THEN Nonbankrupt}\\
Rule15 &\small{IF FF is negative and OP is negative and CR is negative
and MR is average or positive THEN Nonbankrupt}\\
Rule16 &\small{IF FF is negative and OP is negative and CR is negative and
MR is negative THEN Nonbankrupt}\\

\hline

\end{tabular}
}
\caption{\label{IL}Rules generated by an inductive learning method in \cite{kim}}
\end{center}
\end{table}

\begin{table}
\begin{center}
\resizebox{\columnwidth}{!}{
\begin{tabular}{|c|l|}
\hline
\textbf{Assignment rule}&\textbf{Elementary conditions}\\
\hline
Rule1& \small{IF FF is positive and CO is average or positive THEN
Nonbankrupt}\\
Rule2 &\small{IF FF is positive and CO is negative and MR is average
or positive THEN Nonbankrupt}\\
Rule3 &\small{IF FF is positive and CO is negative and MR is negative
THEN Bankrupt}\\
Rule4 &\small{IF FF is average and MR is positive THEN Nonbankrupt}\\
Rule5 &\small{IF FF is average and MR is average and OP is average
or positive THEN Nonbankrupt}\\
Rule6 &\small{IF FF is average and MR is average and OP is negative
THEN Bankrupt}\\
Rule7 &\small{IF FF is average and MR is negative THEN Bankrupt}\\
Rule8 &small{IF FF is negative and OP is positive THEN Nonbankrupt}\\
Rule9 &\small{IF FF is negative and OP is average and IR is average
or positive THEN Nonbankrupt}\\
Rule10 &\small{IF FF is negative and OP is average and IR is negative
THEN Bankrupt}\\
Rule11 &\small{IF FF is negative and OP is negative THEN Nonbankrupt}\\
Rule12 &\small{IF FF is negative and OP is negative and CR is negative and MR is negative THEN Nonbankrupt}\\

\hline

\end{tabular}
}
\caption{\label{NN}Rules generated by a neural networks-based algorithm in \cite{kim}}

\end{center}
\end{table}

\begin{table}
\begin{center}
\resizebox{\columnwidth}{!}{
\begin{tabular}{|c|l|}
\hline
\textbf{Assignment rule}&\textbf{Elementary conditions}\\
\hline
Rule1 & {IF FF is average or positive and CR is average or positive and CO is average or positive THEN Nonbankrupt}\\
Rule2 &{IF FF is negative and CR is negative and CO is negative
and OP is negative THEN Bankrupt}\\
Rule3 &{IF FF is positive and CO is positive THEN Nonbankrupt}\\
Rule4 &{IF IR is average or positive and CR is average or positive
and CO is positive THEN Nonbankrupt}\\
Rule5 &{IF IR is average or positive and MR is average or positive
and FF is average or positive and CO} \\
&{is average or positive and OP is average or positive THEN Nonbankrupt}\\
Rule6 &{IF MR is average or positive and CR is average or positive
and CO is average or positive THEN Nonbankrupt}\\
Rule7 &{IF MR is negative or average and FF is negative and CR
is negative CO is negative THEN Bankrupt}\\
Rule8 &{IF IR is positive and MR is average or positive and CO
is positive THEN Nonbankrupt}\\
Rule9 &{IF IR is average or positive and CO is positive and OP
is average or positive THEN Nonbankrupt}\\
Rule10 &{IF MR is negative and FF is negative and CR is negative and
CO is negative or average and OP is negative or}\\ 
& {average
THEN Bankrupt}\\
Rule11 &{IF IR is negative and MR is negative and FF is negative
and CO is negative THEN Bankrupt}\\

\hline
\end{tabular}
}
\caption{\label{GA}Rules generated by a genetic algorithm in \cite{kim}}
\end{center}

\end{table}
The authors design and apply a genetic algorithm for the induction of a classification system for this problem and compare its quality, for the classification of 232 cases, to those of two other approaches: an inductive learning method and a neural networks algorithm, with satisfying results. Tables \ref{GA}, \ref{IL} and  \ref{NN} respectively describes the classification systems induced by each one of these techniques, and table \ref{res} summarizes their quality indicators (accuracy, coverage, and number of rules as a measure of compactness).

\begin{table}[h!]
\begin{center}

\begin{tabular}{|c|c|c|c|}

\hline
\textbf{Technique}&\textbf{No of Rules}&\textbf{Coverage}&\textbf{Accuracy}\\
\hline
GA &11& 18.5& 94.0\\
Inductive learning &16&15.3& 89.7\\
Neural networks &12&15.6& 90.3\\

\hline
\end{tabular}
\caption{\label{res}The performances of data mining techniques in \cite{kim}}
\end{center}

\end{table}
\begin{table}
\begin{center}
\begin{tabular}{|c|l|}
\hline
\textbf{Assignment rule}&\textbf{Elementary conditions}\\
\hline
Rule1& \small{IF CO is positive THEN
Nonbankrupt}\\
Rule2& \small{IF CO is negative THEN
Bankrupt}\\
Rule3& \small{IF CO is average and FF is negative and CR is negative}\\ 
& {THEN
Bankrupt}\\
Rule3& \small{IF CO is average and FF is not negative or CR is not negative}\\ & {THEN
Nonbankrupt}\\

\hline
\end{tabular}
\caption{\label{our}Reduced form of the genetic algorithm-induced classification system}
\end{center}
\end{table}

By applying the proposed algorithm, we can conclude that the main classification system, based on a genetic algorithm can be highly improved, as evidenced by table \ref{our}, which presents the reduced form of this classification system. Indeed, this reduced form results in full accuracy and coverage, which means that the data set is trivial, as can be verified in \cite{uci}. Moreover, classification systems induced by the other approaches cannot be reduced. This is due to the fact that the assignment rules in these systems include cases of conflict inter-attribute (no attribute is overly important), and intra-attribute (high and low values are represented in each elementary condition). The fact that assignment rules cover these conflicting cases is inversely related to their ability to be reduced. Indeed, it reduces the likelihood of finding mutually exclusive elementary conditions, as they cover a larger range of alternatives. Measuring a classification system's ability to resolve conflicting cases seems however strongly dependent on the definition of attributes, values and classes of the problem at hand. However, analyses of this aspect, such as the existing work for multicriteria aggregation procedures \cite{ROY}, could be conducted for specific rule-based classification systems. 
\section{Conclusion}
In this brief article, we have argued that predictive accuracy and coverage are not sufficient indicators of a classification system's quality, and that any indicator of quality can only be valid if measured for the most reduced possible form of a classification system. Moreover, values for these indicators can mask the triviality of a data set with unnecessary complexity. We have thus proposed a procedure to perform such a reduction. Finally, we have suggested that the reducibility of a classification system seems to be tied to its ability to clarify trade-offs. Investigating this link and building indicators for this ability are some of the possible perspectives for developing this work.

%

%
%


\begin{thebibliography}{5}
%
\bibitem {simon}
P. Langley, and H. A. Simon:
Applications of Machine Learning and Rule Induction.
Communications of the ACM. 55-64(1996)

\bibitem {drsa}
M. Szeląg, S. Greco, and R. Slowinski: 
Variable consistency dominance-based rough set approach to preference learning in multicriteria ranking.
Information Sciences. 277(1), 525-552(2014).

\bibitem {BELACEL}
N. Belacel, H.B. Raval and A.P. Punnen:
Learning multicriteria fuzzy classification method PROAFTN from data.
Computers \& Operations Research, 34(7), 1885-1898(2007).

\bibitem {kim}
M. Kim, and I. Han:
The discovery of experts’ decision rules from qualitative bankruptcy data using genetic algorithms.
Expert Systems with Applications. 25, 637–646(2003). 


\bibitem{Carpineto}
C. Carpineto and G. Romano:
Concept Data Analysis: Theory and Applications.
John Wiley \& Sons, ISBN:0470850558, (2004). 

\bibitem{Garriga}
G.C. Garriga, P. Kralj and N. Lavrac: 
Closed Sets for Labeled Data. 
Journal of Machine Learning Research, 9, 559-580(2008).

\bibitem{Rosa}
R. Meo, P.L. Lanzi and M. Klemettinen:
Database Support for Data Mining Applications: Discovering Knowledge with Inductive Queries. 
Lecture Notes in Computer Science, 2682(2004).

\bibitem{Mario}
M. Boley and H. Grosskreutz: 
Non-redundant Subgroup Discovery Using a Closure System. ECMLPKDD, 179-194(2009).

\bibitem{Ruping}
H. Grosskreutz, D. Paurat and S. Ruping: 
An enhanced relevance criterion for more concise supervised pattern discovery. 
KDD 2012, 1442-1450(2012).

\bibitem{Cohen}
W.W. Cohen: 
Efficient pruning methods for separate-and-conquer rule learning systems.  
IJCAI 1993, 988-994(1993).

\bibitem {uci}
M. Lichman:   
UCI Machine Learning Repository. Irvine, CA: University of California, School of Information and Computer Science (2013) \url{http://archive.ics.uci.edu/ml/datasets/Qualitative_Bankruptcy}.

\bibitem {martin1}
A. Martin, T. M. Lakshmi, and V. P. Venkatesan:
An Analysis on Qualitative Bankruptcy Prediction Rules using Ant-Miner.
International Journal of Computer Applications. 41(21), 0975-8887(2012).

\bibitem {sun}
J. Sun,	Z. Shang, and H. Li:
Imbalance-oriented SVM methods for financial distress prediction: a comparative study among the new SB-SVM-ensemble method and traditional methods.
Journal of the Operational Research Society, 65, 1905–1919(2014). 

\bibitem {cheng}
C. Cheng, C. Chen, and C. Fu:
Financial distress prediction by a radial basis function network with logit analysis learning.
Computers \& Mathematics with Applications, 51(3–4), 579-588(2006).

\bibitem {koklu}
M. Koklu, and K. Tutuncu:
Qualitative Bankruptcy Prediction Rules Using Artificial Intelligence Techniques.
International Conference on challenges in IT, Engineering and Technology (ICCIET’2014), (2014).  

\bibitem {ROY}
B. Roy and V. Mousseau: 
A theoretical framework for analysing the notion of relative importance of criteria. 
Journal of Multicriteria Decision Analysis, 5(2), 145-159(1996).



\end{thebibliography}
\end{document}